\documentclass[journal]{IEEEtran}

\usepackage{setspace,epsfig}
\usepackage{amsfonts,psfrag,amssymb}
\usepackage{color}
\usepackage{cite}
\usepackage{graphicx}

\ifCLASSINFOpdf
\else
\fi

\usepackage[cmex10]{amsmath}
\usepackage{array}

\renewcommand{\star}{*}
\newcommand{\R}{{\mathbb R}}
\newcommand{\cN}{\mathcal{N}} 
 
\newcommand{\tG}{\widetilde{G}}

\newcommand{\beps}{\varepsilon} 
\newcommand{\bx}{{\mathbf x}}
\newcommand{\by}{{\mathbf y}}
\newcommand{\unx}{{\mathbf z}}
\newcommand{\uny}{\by}
\newcommand{\RR}{\R}

\newcommand{\IP}{{\sf IP}}
\newcommand{\LP}{{\sf LP}}
\newcommand{\DL}{{\sf DUAL}}
\newcommand{\CPe}{{\sf CP}$(\beps)$}
\newcommand{\DESCENT}{{\sf DESCENT}}
\newcommand{\EST}{{\sf EST}}

\newcommand{\AL}{{\sf ALGO}}

\newcommand{\lf}{\left}
\newcommand{\rf}{\right}
\newcommand{\bone}{\mathbf{1}}

\newcommand{\beq}{\begin{eqnarray}}
\newcommand{\eeq}{\end{eqnarray}}
\newcommand{\beqn}{\begin{eqnarray*}}
\newcommand{\eeqn}{\end{eqnarray*}}

\newtheorem{theorem}{Theorem}[section]
\newtheorem{lemma}{Lemma}[section]

\newtheorem{example}{Example}[section]

\usepackage{url}

\begin{document}

\title{Message-passing for Maximum Weight Independent Set}
%

\author{Sujay~Sanghavi \qquad Devavrat~Shah \qquad Alan Willsky
\thanks{All authors are with the Department of Electrical Engineering 
and Computer Science,  Massachusetts Institute of Technology, Cambridge, 
MA 02139 USA.  Email: {\tt \{sanghavi,devavrat,willsky\}@mit.edu}}
\thanks{This work was partially supported by NSF CNS-0546590, XXX.}
}

\markboth{Submitted to IEEE Transaction on Information Theory}%
{Shell \MakeLowercase{\textit{et al.}}: Bare Demo of IEEEtran.cls for Journals}

\maketitle

\begin{abstract}

We investigate the use of message-passing algorithms for the problem
of finding the max-weight independent set (MWIS) in a graph. First, we
study the performance of the classical loopy max-product belief
propagation. We show that each fixed point estimate of max-product can
be mapped in a natural way to an extreme point of the LP polytope
associated with the MWIS problem. However, this extreme point may not
be the one that maximizes the value of node weights; the particular
extreme point at final convergence depends on the initialization of
max-product. We then show that if max-product is started from the
natural initialization of uninformative messages, it always solves the
correct LP -- if it converges. This result is obtained via a direct
analysis of the iterative algorithm, and cannot be obtained by looking
only at fixed points.

The tightness of the LP relaxation is thus necessary for max-product
optimality, but it is not sufficient. Motivated by this observation,
we show that a simple modification of max-product becomes gradient
descent on (a convexified version of) the dual of the LP, and
converges to the dual optimum. We also develop a message-passing
algorithm that recovers the primal MWIS solution from the output of
the descent algorithm. We show that the MWIS estimate obtained using
these two algorithms in conjunction is correct when the graph is
bipartite and the MWIS is unique.

Finally, we show that any problem of MAP estimation for probability
distributions over finite domains can be reduced to an MWIS
problem. We believe this reduction will yield new insights and
algorithms for MAP estimation.




\end{abstract}

\section{Introduction}

The max-weight independent set (MWIS) problem is the following: given
a graph with positive weights on the nodes, find the heaviest set of
mutually non-adjacent nodes. MWIS is a well studied combinatorial
optimization problem that naturally arises in many applications. It is
known to be NP-hard, and hard to approximate \cite{Tre}. In this paper
we investigate the use of message-passing algorithms, like loopy
max-product belief propagation, as practical solutions for the MWIS
problem.  We now summarize our motivations for doing so, and then
outline our contribution.

Our primary motivation comes from applications. The MWIS problem
arises naturally in many scenarios involving resource allocation in
the presence of interference. It is often the case that large
instances of the weighted independent set problem need to be (at least
approximately) solved in a distributed manner using lightweight data
structures. In Section \ref{sec:applic} we describe one such
application: scheduling channel access and transmissions in wireless
networks. Message passing algorithms provide a promising alternative
to current scheduling algorithms.

Another, equally important, motivation is the potential for obtaining
new insights into the performance of existing message-passing
algorithms, especially on loopy graphs. Tantalizing connections have
been established between such algorithms and more traditional
approaches like linear programming (see \cite{BSS07},\cite{sanghavi}
\cite{chenweiss} and references therein). We consider MWIS problem to
understand this connection as it provides a rich (it is NP-hard), yet
relatively (analytically) tractable, framework to investigate such
connections.

\subsection{Our contributions}


In Section \ref{sec:prelim} we formally describe the MWIS problem,
formulate it as an integer progam, and present its natural LP
relaxation. We also describe how the MWIS problem arises in wireless
network scheduling.



In Section \ref{sec:max_prod_for_mwis}, we first describe how we
propose using max-product (as a heuristic) for solving the MWIS
problem. Specifically, we construct a probability distribution whose
MAP estimate is the MWIS of the given graph. Max-product, which is a
heuristic for finding MAP estimates, emerges naturally from this
construction. 

Max-product is an iterative algorithm, and is typically executed until
it converges to a fixed point. In Section \ref{sec:fixed_p} we show
that fixed points always exist, and characterize their
structure. Specifically, we show that there is a one-to-one map
between estimates of fixed points, and extreme points of the
independent set LP polytope. This polytope is defined only by the
graph, and each of its extrema corresponds to the LP optimum for a
different node weight function.  This implies that max-product fixed
points attempt to solve (the LP relaxation of) an MWIS problem on the
correct graph, but with different (possibly incorrect) node
weights. This stands in contrast to its performance for the weighted
matching problem \cite{BSS07,sanghavi,bayati-msr}, for which it is
known to {\em always} solve the LP with correct weights.

Since max-product is a deterministic algorithm, the particular fixed
point (if any) that is reached depends on the initialization. In
Section \ref{sec:iter_anal} we pursue an alternative line of analysis,
and directly investigate the performance of the iterative algorithm
itself, started from the ``natural'' initialization of uninformative
messages. Fot this case, we show that max-product estimates exactly
correspond to the {\em true} LP, at all times -- not just the fixed
point.

Max-product bears a striking semantic similarity to dual coordinate
descent on the LP. With the intention of modifying max-product to make
it as powerful as LP, in Section \ref{sec:algo} we develop two
iterative message-passing algorithms. The first, obtained by a minor
modification of max-product, approximately calculates the optimal
solution to the dual of the LP relaxation of the MWIS problem. It does
this via coordinate descent on a convexified version of the dual. The
second algorithm uses this approximate optimal dual to produce an
estimate of the MWIS. This estimate is correct when the original graph
is bipartite. We believe that this algorithm should be of broader
interest.

The above uses of max-product for MWIS involved posing the MWIS as a
MAP estimation problem. In the final Section \ref{sec:map_as_mwis}, we
do the reverse: we show how {\em any} MAP estimation problem on finite
domains can be converted into a MWIS problem on a suitably constructed
auxillary graph. This implies that any algorithm for solving the
independent set problem immediately yields an algorithm for MAP
estimation. This reduction may prove useful from both practical and
analytical perspectives.

\section{Max-weight Independent Set, and its LP Relaxation}\label{sec:prelim}

Consider a graph $G = (V,E)$, with a set $V$ of nodes and a set $E$ of
edges. Let $\cN(i)=\{j \in V: (i,j) \in E\}$ be the neighbors of $i
\in V$.  Positive weights $w_i, i\in V$ are associated with each node.
A subset of $V$ will be represented by vector $\bx = (x_i) \in
\{0,1\}^{|V|}$, where $x_i = 1$ means $i$ is in the subset $x_i = 0$
means $i$ is not in the subset. A subset $\bx$ is called an {\em
independent set} if no two nodes in the subset are connected by an
edge: $(x_i, x_j) \neq (1,1)$ for all $(i,j) \in E$. We are interested
in finding a maximum weight independent set (MWIS) $\bx^*$.  This can
be naturally posed as an integer program, denoted below by \IP.~ The
{\em linear programing relaxation} of \IP~ is obtained by replacing
the integrality constraints $x_i\in \{0,1\}$ with the constraints
$x_i\geq 0$. We will denote the corresponding linear program by
\LP.~The dual of \LP~ is denoted below by \DL.

\begin{eqnarray*}
\IP : & & ~{\sf max} ~~ \sum_{i=1}^n w_i x_i, \\
{\sf s.t.} & & x_i + x_j \leq 1 ~~~~\text{for all}~ (i,j) \in E, \\
& & ~ x_i\in \{0,1\}.
\end{eqnarray*}
\begin{eqnarray*}
\LP : & & ~{\sf max} ~~ \sum_{i=1}^n w_i x_i, \\
{\sf s.t.} & & x_i + x_j \leq 1 ~~~~\text{for all}~ (i,j) \in E, \\
& & ~ x_i\geq 0.
\end{eqnarray*}
\begin{eqnarray*}
\DL : & & {\sf min} ~~ \sum_{(i,j) \in E}   \lambda_{ij}, \\
{\sf s.t.} & &  \sum_{j\in \cN(i)} \lambda_{ij} ~ \geq ~ w_i,
~\text{for all} ~ i \in V, \\ 
& & \lambda_{ij} \geq 0, ~\mbox{for all} ~(i,j) \in E. 
\end{eqnarray*}
It is well-known that \LP~ can be solved efficiently, and if it has an
integral optimal solution then this solution is an MWIS of $G$. If
this is the case, we say that there is no {\em integrality gap}
between \LP~ and \IP~ or equivalently that the \LP~relaxation is {\em
tight}.  

\subsection*{Properties of the \LP}

We now briefly state some of the well-known properties of the MWIS
\LP, as these will be used/referred to in the paper. The polytope of
the \LP~ is the set of feasible points for the linear program. An
extreme point of the polytope is one that cannot be expressed as a
convex combination of other points in the polytope. 

\begin{lemma}\label{lem:halfint}{(\cite{schrijver}, Theorem 64.7)}
The \LP~ polytope has the following properties
\begin{enumerate}
\item For any graph, the MWIS \LP~ polytope is half-integral: any
  extreme point will have each $x_i = 0, 1$ or $\frac{1}{2}$. 
\item For bipartite graps the \LP~ polytope is integral: each extreme
  point will have $x_i = 0$ or 1.
\end{enumerate}
\end{lemma}

Half-integrality is an intriguing property that holds for LP
relaxations of a few combinatorial problems (e.g. vertex cover,
matchings etc.). Half integrality implies that any extremum optimum of
\LP~ will have some nodes set to 1, and all their neighbors set to
0. The nodes set to $\frac{1}{2}$ will appear in clusters: each such
node will have at least one other neighbor also set to
$\frac{1}{2}$. We will see later that a similar structure arises in
max-product fixed points.

\begin{lemma}{(\cite{schrijver}, Corollary 64.9a)}
\LP~ optima are {\em partially correct}: for any graph, any \LP~
optimum $x^*$ and any node $i$, if the mass $x^*_i$ is integral then
there exists an MWIS for which that node's membership is given by
$x^*_i$.
\end{lemma}

The next lemma states the standard complimentary slackness conditions
of linear programming, specialized for the MWIS \LP, and for the case
when there is no integrality gap.
  
\begin{lemma}\label{lem:cs}
  When there is no integrality gap between \IP~ and \LP,~there
  exists a pair of optimal solutions $\bx=(x_i)$, $\lambda =
  (\lambda_{ij})$ of \LP~and \DL ~respectively, such that: (a) $\bx
  \in \{0,1\}^n$, (b) $x_i \lf(\sum_{j \in \cN(i)} \lambda_{ij} -
  w_i\rf) = 0$ for all $i \in V$, (c) $\lf(x_i + x_j - 1\rf)
  \lambda_{ij} = 0$, for all $(i,j) \in E$.
\end{lemma}

\subsection{Sample Application: Scheduling in Wireless Networks}
  \label{sec:applic} 

We now briefly describe an important application that requires an
efficient, distributed solution to the MWIS problem: transmision
scheduling in wireless networks that lack a centralized
infrastructure, and where nodes can only communicate with local
neighbors (e.g. see \cite{JS07}). Such networks are ubiquitous 
in the modern world: examples range from sensor networks that 
lack wired connections to the fusion center, and ad-hoc networks 
that can be quickly deployed in areas without coverage, to the 
802.11 wi-fi networks that currently represent the most widely 
used method for wireless data access.

Fundamentally, any two wireless nodes that transmit at the same time
and over the same frequencies will interfere with each other, if they
are located close by. Interference means that the intended receivers
will not be able to decode the transmissions. Typically in a network
only certain pairs of nodes interfere. The scheduling problem is to
decide which nodes should transmit at a given time over a given
frequency, so that {\em (a)} there is no interference, and {\em (b)}
nodes which have a large amount of data to send are given priority. In
particular, it is well known that if each node is given a {\em weight}
equal to the data it has to transmit, optimal network operation
demands scheduling the set of nodes with highest total weight. If a ``
conflict graph'' is made, with an edge between every pair of
interfering nodes, the scheduling problem is {\em exactly} the problem
of finding the MWIS of the conflict graph. The lack of an
infrastructure, the fact that nodes often have limited capabilities,
and the local nature of communication, all necessitate a lightweight
distributed algorithm for solving the MWIS problem.

\section{Max-product for MWIS}\label{sec:max_prod_for_mwis}

The classical max-product algorithm is a heuristic that can be used to
find the MAP assignment of a probability distribution. Now, given an
MWIS problem on $G=(V,E)$, associate a binary random variable $X_i$
with each $i \in V$ and consider the following joint distribution: for
$\bx \in \{0,1\}^n$,
\begin{eqnarray}
p\left(\bx\right ) & = & \frac{1}{Z} \prod_{(i,j) \in E}
\mathbf{1}_{\{x_i + x_j \leq 1\}} \prod_{i \in V} \exp(w_i x_i),
\label{e1}
\end{eqnarray}
where $Z$ is the normalization constant. In the above, $\mathbf{1}$ is
the standard indicator function: $\bone_{\text{true}} = 1$ and
$\bone_{\text{false}}=0$. It is easy to see that $p(\bx) =
\frac{1}{Z}\exp\lf(\sum_i w_i x_i\rf)$ if $\bx$ is an independent set,
and $p(\bx) = 0$ otherwise.  Thus, any MAP estimate $\arg\max_{\bx}
p(\bx)$ corresponds to a maximum weight independent set of $G$.

The update equations for max-product can be derived in a standard and
straightforward fashion from the probability distribution. We now
describe the max-product algorithm as derived from $p$.  At every
iteration $t$ each node $i$ sends a {\em message} $\{m^{t}_{i\to
j}(0),m^{t}_{i\to j}(1) \}$ to each neighbor $j\in \cN(i)$.  Each node
also maintains a {\em belief} $\{b_i^t(0),b_i^t(1)\}$ vector. The
message and belief updates, as well as the final output, are computed
as follows.

\vspace{.05in}
\noindent{\bf Max-product for MWIS}
\vspace{.05in}
\hrule
\vspace{.1in}
\begin{itemize}
\item[(o)] Initially, $m^{0}_{i\to j}(0) = m^{0}_{j\to i}(1) = 1$ for
  all $(i,j) \in E$.
\item[(i)] The messages are updated as follows:
  \begin{eqnarray*}
    m_{i\rightarrow j}^{t+1}(0) & = & \max \left \{ \prod_{k\neq j, k
      \in \cN(i)} m^t_{k \rightarrow i} (0) \, , ~~ \right. \\
      & ~ & \qquad \qquad \left. \, e^{w_i}
      \prod_{k\neq j, k\in \cN(i)} m^t_{k \rightarrow i} (1)\right \},
      \\ m_{i\rightarrow j}^{t+1}(1) & = & \prod_{k\neq j, k\in
      \cN(i)} m^t_{k\rightarrow i} (0).
  \end{eqnarray*}
  
\item[(ii)] Nodes $i \in V$, compute their beliefs as follows: 
\begin{eqnarray*}
b^t_i(0) & = & \prod_{k \in \cN(i)} m^t_{k \rightarrow i} (0), \\
b^t_i(1) & = & e^{w_i} \prod_{k\in \cN(i)} m^t_{k \rightarrow i} (1). 
\end{eqnarray*}

\item[(iii)] Estimate max. wt. independent set $\bx(b^{t+1})$ as
follows:
\begin{eqnarray*}
x_i(b_i^t) = 1 & \text{if} & b^t_i(1) > b^t_i(0) \\
x_i(b_i^t) = 0 & &  b^t_i(1) < b^t_i(0) \\
x_i(b_i^t) = ? & &  b^t_i(1) = b^t_i(0) 
\end{eqnarray*}

\item[(iv)] Update $t = t + 1$; repeat from (i) till $\bx(b^{t})$
converges and output the converged estimate.

\end{itemize}
\vspace{.1in}
\hrule
\vspace{.1in} 

For the purpose of analysis, we find it convenient to transform the
messages and their dynamics as follows. First, define
$$\gamma^t_{i\rightarrow j} = \log \left (
\frac{m^t_{i\rightarrow j}(0)}{m^t_{i\rightarrow j}(1)} \right
).$$ 
Here, since the algorithm starts with all messages being strictly 
positive, the messages will remain strictly positive over any 
finite number of iterations. Therefore, taking logarithm is a 
valid operation. With this new definition, step (i) of the max-product 
becomes 
\begin{eqnarray}
\gamma^{t+1}_{i\rightarrow j} = \left ( w_i - \sum_{k \in
\cN(i)-j}\gamma^t_{k\rightarrow i} \right )_+, \label{eq:min-sum}
\end{eqnarray} 
where we use the notation $(x)_+ = \max\{x,0\}$. The final estimation
step (iii) of max-product takes the following form:
\begin{eqnarray}
x_i(\gamma^t) =1 & \text{if} & w_i > \sum_{k\in \cN(i)}
  \gamma^t_{k\rightarrow i} \label{eq:est_1_def} \\ x_i(\gamma^t) =0 &
  & w_i < \sum_{k\in \cN(i)} \gamma^t_{k\rightarrow i}
  \label{eq:est_0_def} \\
  x_i(\gamma^t) =? & & w_i = \sum_{k\in \cN(i)} \gamma^t_{k\rightarrow
  i} \label{eq:est_?_def}
\end{eqnarray} 
This modification of max-product is often known as the ``min-sum''
algorithm, and is just a reformulation of the max-product.  In the
rest of the paper we refer to this as simply the max-product
algorithm.

\section{Fixed Points of Max-product} \label{sec:fixed_p}

When applied to general graphs, max product may either {\em (a)} not
converge, {\em (b)} converge, and yield the correct answer, or {\em
(c)} converge but yield an incorrect answer. Characterizing when each
of the three situations can occur is a challenging and important
task. One approach to this task has been to look directly at the fixed
points, if any, of the iterative procedure (see e.g. \cite{YFW00}). In
this section we investigate properties of fixed points, by formally
establishing a connection to the \LP~ polytope.

Note that a set of messages $\gamma^*$ is a fixed point of max-product
if, for all $(i,j)\in E$
\begin{equation}
\gamma^*_{i\rightarrow j} = \left ( w_i - \sum_{k \in \cN(i)-
  j}\gamma^*_{k\rightarrow i} \right )_+ \label{eq:fixedp}
\end{equation}
The following lemma establishes that fixed points always exist.

\begin{lemma}\label{lem:fixedpex}
There exists at least one fixed point $\gamma^*$ such that
$\gamma^*_{i\to j} \in [0,w_i]$ for each $(i,j)\in E$
\end{lemma}
\begin{proof}
Let $w^* = \max_i w_i$, and suppose at time $t$ each $\gamma_{i\to
j}^t \in [0,w^*]$. From (\ref{eq:min-sum}) it is clear that this will
result in the messages $\gamma^{t+1}$ at the next time also having
each $\gamma_{i\rightarrow j}^{t+1} \in [0,w^*]$. Thus, the
max-product update rule (\ref{eq:min-sum}) maps a message vector
$\gamma^t \in [0,w^*]^{2|E|}$ into another vector in
$[0,w^*]^{2|E|}$. Also, it is easy to see that (\ref{eq:min-sum}) is a
continuous function. Therefore, by Brouwer's fixed point theorem there
exists a fixed point $\gamma^* \in [0,w^*]^{2|E|}$.
\end{proof}


We now study properties of the fixed points in order to understand the
correctness of the estimate output by max-product. The following
theorem characterizes the structure of estimates at
fixed-points. Recall that the estimate $x_i(\gamma^*)$ for node $i$
can be 0,1 or ?.

\begin{theorem}\label{thm:fp}
Let $\gamma^*$ be a fixed point, and let $\bx(\gamma^*) =
(x_i(\gamma^*))$ be the corresponding estimate. Then,
\begin{enumerate}
\item If $x_i(\gamma^*) = 1$ then every neighbor $j\in \cN(i)$ has
  $x_j(\gamma^*) = 0$.
\item If $x_i(\gamma^*) = 0$ then at least one neighbor $j\in \cN(i)$
  has $x_j(\gamma^*) = 1$.
\item If $x_i(\gamma^*) = ?$ then at least one neighbor $j\in \cN(i)$
  has $x_j(\gamma^*) = ?$.
\end{enumerate}
\end{theorem}

Before proving Theorem \ref{thm:fp} we discuss its implications.
Recall from Lemma \ref{lem:halfint} that every extreme point of the
\LP~ polytope consists of each node having a value of 0,1 or
$\frac{1}{2}$. If all weights are positive, the optimum of \LP~ will
have the following characteristics: every node with value 1 will be
surrounded by nodes with value 0, every node with value 0 will have at
least one neighbor with value 1, and every node with value
$\frac{1}{2}$ will have one neighbor with value $\frac{1}{2}$. These
properties bear a remarkable similarity to those in Theorem
\ref{thm:fp}. Indeed, given a fixed point $\gamma^*$ and its estimates
$\bx(\gamma^*)$, make a vector $\by$ by setting

\begin{tabular}{ccc}
$y_i = \frac{1}{2}$ & if estimate for $i$ is & $x_i(\gamma^*) = ?$ \\
$y_i = 1$ & & $x_i(\gamma^*) = 1$ \\
$y_i = 0$ & & $x_i(\gamma^*) = 0$
\end{tabular}

Then, Theorem \ref{thm:fp} implies that $\by$ will be an extreme point
of the \LP~ polytope, and also one that maximizes {\em some} weight
function consisting of positive node weights. Note however that this
{\em may not} be the true weights $w_i$. In other words, given any
MWIS problem with graph $G$ and weights $w$, each max-product fixed
point represents the optimum of the LP relaxation of some MWIS problem
on the same graph $G$, but possibly with different weights
$\widehat{w}$.

The fact that max-product estimates optimize a different weight
function means that both eventualities are possible: \LP~ giving the
correct answer but max-product failing, and vice versa. We now provide
simple examples for each one of these situations.

\begin{figure}[htb]
\begin{center}
\epsfig{file=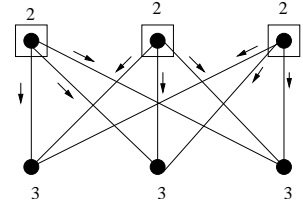,width=1.5in} 
\end{center}
\label{fig:ex1}
\caption{This example shows that max-product fixed point may result
in-correct answer even though LP is tight.}
\end{figure}

\begin{figure}[htb]
\begin{center}
\epsfig{file=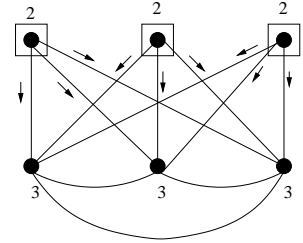,width=1.5in}
\end{center}
\label{fig:ex2}
\caption{This example shows that max-product fixed point can find
right MWIS even though LP relaxation is not tight.}
\end{figure}

The Figures \ref{fig:ex1} and \ref{fig:ex2} present graphs and the
corresponding fixed points of max-product.  In each graph, numbers
represent node weights, and an arrow from $i$ to $j$ represents a
message value of $\gamma^*_{i\rightarrow j} = 2$.  All other messages,
which do not have arrows, have value zero.  The boxed nodes indicate
the ones for which the estimate $x_i(\gamma^*) =1$. It is easy to
verify that both examples represent max-product fixed points.

For the graph in Figure \ref{fig:ex1}, the max-product 
fixed point results in an incorrect estimate. However, the 
graph is bipartite, and hence \LP~ will provide the correct 
answer. For the graph in Figure \ref{fig:ex2}, there is an
integrality gap between \LP~ and \IP:~ setting each 
$x_i = \frac{1}{2}$ yields an optimal value of 7.5 for \LP,~ 
while the optimal solution to \IP~ has value 6. Note that
the estimate at the fixed point of max-product is
the correct MWIS. It is also worth noticing that both of 
these examples, the fixed points lie in the strict 
interiors of a non-trivial region of attraction: starting
the iterative procedure from within these regions will result in
convergence to the corresponding fixed point. 
These examples indicate that it may not be possible to resolve the
question of relative strength of the two procedures based solely on an
analysis of the fixed points of max-product.

The particular fixed point, if any, that max-product converges to
depends on the initialization of the messages; each fixed point will
have its own region of convergence. In Section \ref{sec:iter_anal} we
directly analyze the iterative algorithm when started from the
``natural'' initialization of unbiased messages. As a byproduct of
this analysis, we prove that if max-product from this initialization
converges, then the resulting fixed-point estimate is the optimum of
\LP; thus, in this case the max-product fixed point solves the
``correct'' \LP.

\begin{proof}[Proof of Theorem \ref{thm:fp}]
The proof of Theorem \ref{thm:fp} follows from manipulations of the
fixed point equations (\ref{eq:fixedp}). For ease of notation we
replace $\gamma^*$ by $\gamma$. We first prove the following
statements on how the estimates determine the relative ordering of the
two messages (one in each direction) on any given edge:
\begin{eqnarray}
x_i(\gamma) = 1 & \Rightarrow & \gamma_{i\rightarrow j} >
\gamma_{j\rightarrow i}\quad \forall j\in \cN(i) \label{eq:est_1} \\
x_i(\gamma) = ? & \Rightarrow & \gamma_{i\rightarrow j} =
\gamma_{j\rightarrow i}\quad \forall j\in \cN(i) \label{eq:est_?} 
\end{eqnarray}
The above equations cover every case except for edges between two
nodes with 0 estimates. This is covered by the following
\begin{eqnarray}
x_i(\gamma) = 0 ~\text{and} ~ x_j(\gamma) = 0 & \Rightarrow &
\gamma_{i\rightarrow j} = \gamma_{j\rightarrow i} = 0 \label{eq:est_0}
\end{eqnarray}

Suppose first that $i$ is such that $x_i(\gamma^*) = 1$. By definition
(\ref{eq:fixedp}) of the fixed point,
\[
\gamma_{i\rightarrow j} ~ \geq ~ w_i - \sum_{k\in \cN(i)-j}
\gamma_{k\rightarrow i}
\]
However, by (\ref{eq:est_1_def}), the fact that $x_i(\gamma) = 1$
implies that
\[
w_i - \sum_{k\in \cN(i)-j} \gamma_{k\rightarrow i} ~ > ~
\gamma_{j\rightarrow i}
\]
Putting the above two equations together proves (\ref{eq:est_1}). The
proof of (\ref{eq:est_?}) is along similar lines. Suppose now $i$ is
such that $x_i(\gamma) = ?$. By (\ref{eq:est_?_def}) this implies
that $w_i = \sum_{k\in \cN(i)} \gamma_{k\rightarrow i}$, and so from
(\ref{eq:fixedp}) we have that
\[
\gamma_{i\rightarrow j} ~ = ~ w_i - \sum_{k\in \cN(i)-j}
\gamma_{k\rightarrow i}
\]
Also, the fact that $x_i(\gamma) = ?$ means that
\[
w_i - \sum_{k\in \cN(i)-j} \gamma_{k\rightarrow i} ~ = ~
\gamma_{j\rightarrow i}
\]
Putting the above two equations together proves (\ref{eq:est_?}). We
now prove the three parts of Theorem \ref{thm:fp}.

{\em Proof of Part 1):} Let $i$ have estimate $x_i(\gamma) = 1$, and
suppose there exists a neighbor $j\in \cN(i)$ such that $x_j(\gamma)
= ?$ or 1. Then, from (\ref{eq:est_1}) it follows that
$\gamma_{i\rightarrow j} > \gamma_{j\rightarrow i}$, and from
(\ref{eq:est_?}) it further follows that $\gamma_{i\rightarrow j}
\leq \gamma_{j\rightarrow i}$. However, this is a contradiction, and
thus every neighbor of $i$ has to have estimate 0.

{\em Proof of Part 2):} Let $i$ have estimate $x_i(\gamma) = 0$. Since
$w_i \geq 0$, (\ref{eq:est_0_def}) implies that there exists at least
one neighbor $j\in \cN(i)$ such that the message $\gamma_{j\rightarrow
i} >0$. From (\ref{eq:est_0}), this means that the estimate
$x_j(\gamma)$ cannot be 0. Suppose now that $x_j(\gamma)=?$. From
(\ref{eq:est_1}) it follows that $\gamma_{i\rightarrow j} =
\gamma_{j\rightarrow i} > 0$, and so
\[
\gamma_{i\rightarrow j} ~ = ~ w_i - \sum_{k\in \cN(i)-j}
\gamma_{k\rightarrow i}
\]
However, since $\gamma_{i\rightarrow j} = \gamma_{j\rightarrow i}$,
this means that 
\[
\gamma_{j\rightarrow i} ~ = ~ w_i - \sum_{k\in \cN(i)-j}
\gamma_{k\rightarrow i}
\]
which violates (\ref{eq:est_0_def}), and thus the assumption that
$x_i(\gamma) = 0$. Thus it has to be that $x_i(\gamma) = 1$.

{\em Proof of Part 3):} Let $i$ have estimate $x_i(\gamma) = ?$. Since
$w_i \geq 0$, (\ref{eq:est_?_def}) implies that there exists at least
one neighbor $j\in \cN(i)$ such that the message $\gamma_{j\rightarrow
i} >0$. From (\ref{eq:est_?}) it follows that
\[
\gamma_{i\rightarrow j} = \gamma_{j\rightarrow i} = w_j - \sum_{l \neq
  i} \gamma_{l\rightarrow j}
\]
Thus $w_j = \sum_l \gamma_{l\rightarrow j}$, which by
(\ref{eq:est_?_def}) means that $x_j(\gamma) = ?$. Thus $i$ has at
least one neighbor $j$ with estimate $x_j(\gamma) = ?$.
\end{proof}

\section{Direct Analysis of the Iterative Algorithm}
\label{sec:iter_anal}

In the last section, we saw that fixed points of Max-product may
correspond to optima ``wrong'' linear programs: ones that operate on
the same feasible set as \LP, but optimize a different linear
function. However, there will also be fixed points that correspond to
optimizing the correct function. Max-product is a deterministic
algorithm, and so which of these fixed points (if any) are reached is
determined by the initialization. In this section we directly analyze
the iterative algorithm itself, as started from the ``natural''
initialization $\gamma = 0$, which corresponds to uninformative
messages

We show that the resulting estimates are characterized by optima of
the {\em true} \LP, at every time instant (not just at fixed
points). This implies that, if a fixed point is reached, it will
exactly reflect an optimum of \LP. Our main theorem in this section is
stated below.

\begin{theorem}\label{thm:lp_mp}
Given any MWIS problem on weighted graph $G$, suppose max-product is
started from the initial condition $\gamma = 0$. Then, for any node
$i\in G$.
\begin{enumerate}
\item If there exists {\em any} optimum $x^*$ of \LP~ for which the
  mass assigned to $i$ satisfies $x^*_i <1$, then the max-product
  estimate $x_i(\gamma^t)$ is 0 or ? for all {\em even} times
  $t$.
\item If there exists {\em any} optimum $x^*$ of \LP~ for which the
  mass assigned to edge $i$ satisfies $x^*_i >0$, then the max-product
  estimate $x_i(\gamma^t)$ is 1 or ? for all {\em odd} times $t$.
\end{enumerate}
\end{theorem}

From the above theorem, it is easy to see what will happen if \LP~ has
non-integral optima. Suppose node $i$ is assigned non-integral mass at
{\em some} \LP~ optimum $x^*$. This implies that $i$ and $x^*$ will
satisfy both parts of the above theorem. The estimate at node $i$ will
thus either keep varying every alternate time slot, or will converge
to ?. Either way, max-product will fail to provide a useful estimate
for node $i$.

Theorem \ref{thm:lp_mp} also reveals further insights into the
max-product estimates. Suppose for example the estimates converge to
informative answers for a subset of the nodes. Theorem \ref{thm:lp_mp}
implies that every \LP~ optimum assigns the same integral mass to any
fixed node in this subset, and that the converged estimate is the same
as this mass.

The proof of this theorem relies on the computation tree
interpretation of max-product estimates. We now specify this
interpretation for our problem, and then prove Theorem
\ref{thm:lp_mp}.

\subsection*{Computation Tree for MWIS}

The proof of Theorem \ref{thm:lp_mp} relies on the computation tree
interpretation \cite{bib:weiss_local_opt,bib:tatik_jordan} of the
loopy max-product estimates. In this section we briefly outline this
interpretation.  For any node $i$, the {\em computation tree} at time
$t$, denoted by $T_i(t)$, is defined recursively as follows: $T_i(1)$
is just the node $i$. This is the {\em root} of the tree, and in this
case is also its only leaf. The tree $T_i(t)$ at time $t$ is generated
from $T_i(t-1)$ by adding to each leaf of $T_i(t-1)$ a copy of each of
its neighbors in $G$, {\em except for the one neighbor that is already
present in $T_i(t-1)$}. Each node in $T_i$ is a copy of a node in $G$,
and the weights of the nodes in $T_i$ are the same as the
corresponding nodes in $G$. The computation tree interpretation is
stated in the following lemma.

\begin{lemma}\label{lem:comp_tree}
For any node $i$ at time $t$,
\begin{itemize}
\item $x_i(\gamma^t) = 1$ if and only if the root of $T_i(t)$ is a
  member of {\em every} MWIS on $T_i(t)$.
\item $x_i(\gamma^t) = 0$ if and only if the root of $T_i(t)$ is not a
  member of {\em any} MWIS on $T_i(t)$.
\item $x_i(\gamma^t) = ?$ else.
\end{itemize}
\end{lemma}

Thus the max-product estimates correspond to max-weight independent
sets on the computation trees $T_i(t)$, as opposed to on the original
graph $G$. 

{\bf Example:} Consider figure \ref{fig:mwis_comp_tree}. On the left
is the original loopy graph $G$. On the right is $T_a(4)$, the
computation tree for node $a$ at time 4.

\begin{figure}[h]
\centering{\epsfig{file=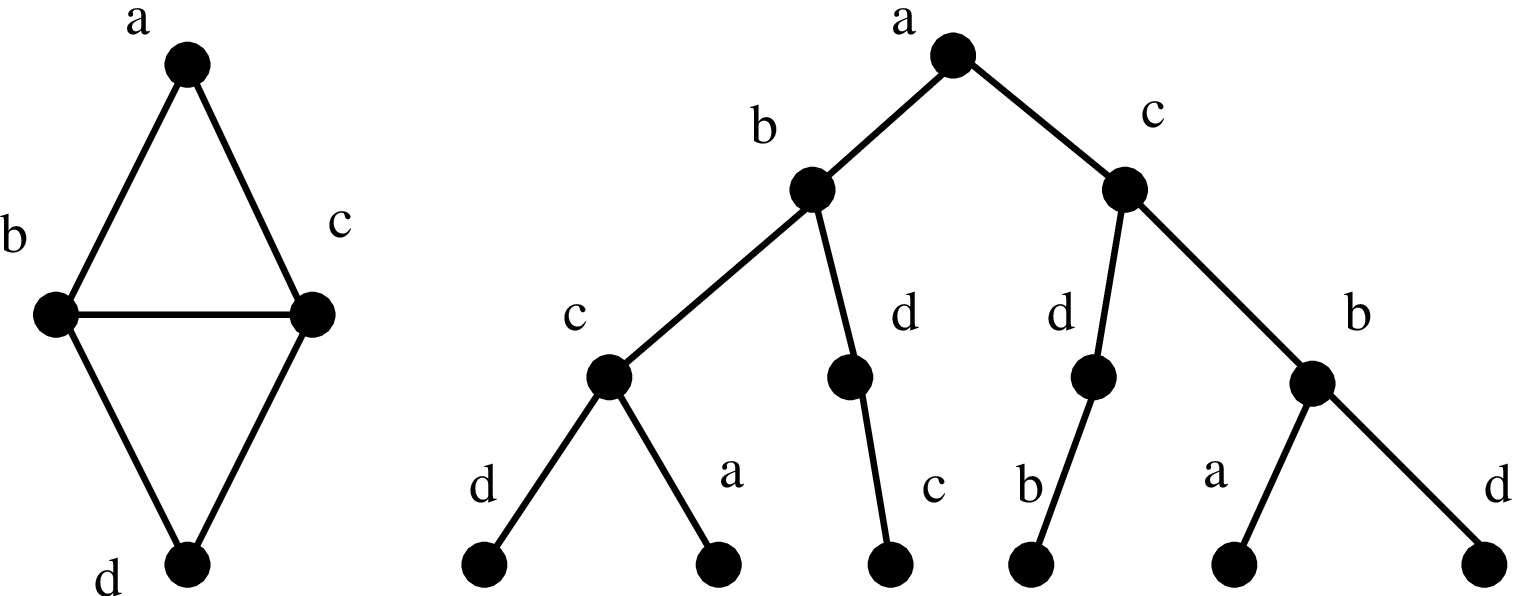,height = 1.0in}}
\label{fig:mwis_comp_tree}
\end{figure}

\subsection*{Proof of Theorem \ref{thm:lp_mp}}

We now prove Theorem \ref{thm:lp_mp}. For brevity, in this proof we
will use the notation $\hat{x}^t_i = x_i(\gamma^t)$ for the estimates.
Suppose now that part 1 of the theorem is not true, i.e. there exists
node $i$, an optimum $x^*$ of \LP~ with $x^*_i>0$, and an odd time $t$
at which the estimate is $\hat{x}^t_i = 0$. Let $T_i(t)$ be the
corresponding computation tree. Using Lemma \ref{lem:comp_tree} this
means that the root $i$ is {\em not} a member of any MWIS of
$T_i(t)$. Let $I$ be some MWIS on $T_i(t)$. We now define the
following set of nodes
\begin{equation*}
I^* = \left \{ j\in T_i(t) \,:\,j\notin I, ~\text{and copy of $j$ in
  $G$ has $x^*_{j}>0$}  \right \} \label{eq:istar}
\end{equation*}
In words, $I^*$ is the set of nodes in $T_i(t)$ which are not in
$I$, and whose copies in $G$ are assigned strictly positive mass by
the LP optimum $x^*$.

Note that by assumption the root $i\in I^*$ and $i\notin I$. Now, from
the root, recursively build a {\em maximal alternating subtree} $S$ as
follows: first add root $i$, which is in $I^*-I$. Then add all
neighbors of $i$ that are in $I-I^*$. Then add all {\em their}
neighbors in $I^*-I$, and so on. The building of $S$ stops either when
it hits the bottom level of the tree, or when no more nodes can be
added while still maintaining the alternating structure.  Note the
following properties of $S$:
\begin{itemize}
\item $S$ is the disjoint union of $(S\cap I)$ and $(S\cap I^*)$.
\item For every $j\in S\cap I$, all its neighbors in $I^*$ are
  included in $S\cap I^*$. Similarly for every $j\in S\cap I^*$, all
  its neighbors in $I$ are included in $S\cap I$.
\item Any edge $(j,k)$ in $T_i(t)$ has at most one endpoint in $(S\cap
  I)$, and at most one in $(S\cap I^*)$.
\end{itemize}
We now state a lemma, which we will prove later. The proof uses the
fact that $t$ is odd.

\begin{lemma}\label{lem:set_flip_lp_mp}
  The weights satisfy $w(S\cap I) \leq w(S\cap I^*)$.
\end{lemma}

We now use this lemma to prove the theorem. Consider the set $I'$
which changes $I$ by flipping $S$:
\[
I' = I - (S\cap I) + (S\cap I^*)
\]
We first show that $I'$ is also an independent set on $T_i(t)$. This
means that we need to show that every edge $(j,k)$ in $T_i(t)$ touches
at most one node in $I'$. There are thus three possible scenarios for
edge $(j,k)$:
\begin{itemize}
\item $j,k \notin S$. In this case, membership of $j,k$ in $I'$ is the
  same as in $I$, which is an independent set. So $(j,k)$ has at most
  one node touching $I'$.
\item One node $j\in S\cap I$. In this case, $j\notin I'$, and hence
  again at most one of $j,k$ belongs to $I'$.
\item One node $k\in S\cap I^*$ but other node $j\notin S\cap I$. This
  means that $j\notin I$, because every neighbor of $k$ in $I$ should
  be included in $S\cap I$. This means that $j\notin I'$, and hence
  only node $k\in I'$ for edge $(j,k)$.
\end{itemize}
Thus $I'$ is an independent set on $T_i(t)$. Also, by Lemma
\ref{lem:set_flip_lp_mp}, we have that
\[
w(I')~ \geq ~ w(I)
\]
However, $I$ is an MWIS, and hence it follows that $I'$ is also an
MWIS of $T_i(t)$. However, by construction, root $i\in I'$, which
violates the fact that $\hat{x}_i(t) = 0$. The contradiction is thus
established, and Part 1 of the theorem is proved. Part 2 is proved in
a similar fashion. \hfill $\blacksquare$


\noindent {\em Proof of Lemma \ref{lem:set_flip_lp_mp}:}

The proof of this lemma involves a perturbation argument on the
\LP. For each node $j\in G$, let $m_{j}$ denote the number of times
$j$ appears in $S\cap I$ and $n_{j}$ the number of times it appears in
$S\cap I^*$. Define
\begin{equation}
x ~ = ~ x^* + \epsilon (m - n) \label{eq:x}
\end{equation}
We now show state a lemma that is proved immediately following this
one.

\begin{lemma} \label{lem:feas}
$x$ is a feasible point for \LP, for small enough $\epsilon$.
\end{lemma}

We now use this lemma to finish the proof of Lemma
\ref{lem:set_flip_lp_mp}. Since $x^*$ is an optimum of \LP, it follows
that $w'x \leq w'x^*$, and so $w'm \leq w'n$. However, by definition,
$w'm = w(S\cap I)$ and $w'n = w(S\cap I^*)$. This finishes the
proof. \hfill $\blacksquare$

\noindent {\em Proof of Lemma \ref{lem:feas}:}

We now show that this $x$ as defined in (\ref{eq:x}) is a feasible
point for \LP, for small enough $\epsilon$. To do so we have to check
node constraints $ x_{j}\geq 0$ and edge constraints $x_j + x_k \leq
1$ for every edge $(j,k)\in G$. Consider first the node
constraints. Clearly we only need to check them for any $j$ which has
a copy $j\in I^*\cap S$. If this is so, then by the definition
(\ref{eq:istar}) of $I^*$, $x_{j}^* >0$. Thus, for any $m_{j}$ and
$n_{j}$, making $\epsilon$ small enough can ensure that $x_{j}^* +
\epsilon (m_{j} - n_{j}) \geq 0$.

Before we proceed to checking the edge constraints, we make two
observations. Note that for any node $j$ in the tree, $j\in S\cap I$
then
\begin{itemize}
\item $x^*_j<1$, i.e. the mass $x^*_j$ put on $j$ by the \LP~ optimum
  $x^*$ is {\em strictly} less than 1. This is because of the
  alternating way in which the tree is constructed: a node $j$ in the
  tree is included in $S\cap I$ {\em only if} the parent $p$ of $j$ is
  in $S\cap I^*$ (note that the root $i\in S\cap I^*$ by
  assumption). However, from the definition of $I^*$, this means that
  $x^*_p>0$, i.e. the parent has positive mass at the \LP~ optimum
  $x^*$. This means that $x_j^* < 1$, as having $x_j^* = 1$ would mean
  that the edge constraint $x^*_p + x^*_j \leq 1$ is violated.
\item $j$ is not a leaf of the tree. This is because $S$ alternates
  between $I$ and $I^*$, and starts with $I^*$ at the root in level 1
  (which is odd). Hence $S\cap I$ will occupy even levels of the tree,
  but the tree has odd depth (by assumption $t$ is odd).
\end{itemize}
Now consider the edge constraints.  For any edge $(j,k)$, if the \LP~
optimum $x^*$ is such that the constraint is loose -- i.e. if $x^*_j +
x^*_k <1$ -- then making $\epsilon$ small enough will ensure that $x_j
+ x_k \leq 1$. So we only need to check the edge constraints which are
tight at $x^*$.

For edges with $x^*_j + x^*_k =1$, every time any copy of one of the
nodes $j$ or $k$ is included in $S\cap I$, the other node is included
in $S\cap I^*$. This is because of the following: if $j$ is included
in $S\cap I$, and $k$ is its parent, we are done since this means
$k\in S\cap I^*$. So suppose $k$ is not the parent of $j$. From the
above it follows that $j$ is not a leaf of the tree, and hence $k$
will be one of its children. Also, from above, the mass on $j$
satisfies $x^*_j <1$. However, by assumption $x^*_j + x^*_k =1$, and
hence the mass on $k$ is $x^*_k >0$. This means that the child $k$ has
to be included in $S\cap I^*$.

It is now easy to see that the edge constraints are satisfied: for
every edge constraint which is tight at $x^*$, every time the mass on
one of the endpoints is increased by $\epsilon$ (because of that node
appearing in $S\cap I$), the mass on the other endpoint is decreased
by $\epsilon$ (because it appears $S\cap I^*$). \hfill $\blacksquare$

\section{A Convergent message-passing algorithm}\label{sec:algo}

In Section \ref{sec:iter_anal} we saw that max-product started from
the natural initial condition solves the correct \LP~ at the fixed
point, {\em if it converges}. However, convergence is not guaranteed,
indeed it is quite easy to construct examples where it will not
converge. In this section we present a convergent message-passing
algorithm for finding the MWIS of a graph. It is based on modifying
max-product by drawing upon a dual co-ordinate descent and the barrier
method. The algorithm retains the iterative and distributed nature of
max-product. The algorithm operates in two steps, as described below.

\vspace{.05in}
\noindent\AL($\beps, \delta, \delta_1$)
\vspace{.05in}
\hrule 
\vspace{.05in}
\begin{itemize}
\item[(o)] Given an MWIS problem, and (small enough) positive
parameters $\beps, \delta$, run sub-routine \DESCENT($\beps,\delta$)
to obtain an output $\lambda^{\beps,\delta} = (\lambda^{\beps,
\delta}_{ij})_{(i,j)\in E}$ 
$\lambda^{\beps,\delta}$ is an approximate dual of the MWIS problem.
\item[(i)] Next, using (small enough) $\delta_1 >0$, use 
\EST($\lambda^{\beps, \delta}, \delta_1$), to produce an 
estimate for the MWIS as an output of the algorithm. 
\end{itemize}
\vspace{.05in}
\hrule ~~\\
\vspace{.1in}

Next, we describe \DESCENT~ and \EST, state their properties and then
combine them to produce the following result about the convergence,
correctness and bound on convergence time for the overall algorithm.

\subsection{\DESCENT: algorithm}

Here, we describe the \DESCENT~ algorithm. It is influenced by the
max-product and dual coordinate descent algorithm for \DL.~First,
consider the standard coordinate descent algorithm for \DL.~It
operates with variables $\{\lambda_{ij},(i,j)\in E\}$ (with notation
$\lambda_{ij}=\lambda_{ji}$). It is an iterative procedure; in each
iteration $t$ one edge $(i,j)\in E$ is picked\footnote{Edges can be
picked either in round-robin fashion, or uniformly at random.} and
updated
\begin{eqnarray}
\lambda^{t+1}_{ij}  & = &  \max\lf\{ 0, \lf ( w_i - \sum_{k \in
  \cN(i), k \neq j} \lambda^t_{ik}\rf ) ~ , \rf. \nonumber \\
  & ~ & \qquad \qquad \lf. \lf ( w_j - \sum_{k
  \in \cN(j), k \neq i} \lambda^t_{jk}\rf ) \rf\}. \label{eq:dco}
\end{eqnarray}
The $\lambda$ on all the other edges remain unchanged from $t$ to
$t+1$.  Notice the similarity (at least syntactic) between standard
dual coordinate descent (\ref{eq:dco}) and max-product
(\ref{eq:min-sum}). In essence, the dual coordinate descent can be
thought of as a {\em sequential bidirectional} version of the
max-product algorithm. 

Since, the dual coordinate descent algorithm is designed so that at
each iteration, the cost of the \DL~is non-increasing, it always
converges in terms of the cost. However, the converged solution may
not be optimum because \DL~ contains the ``non-box'' constraints
$\sum_{j \in \cN(i)} \lambda_{ij} \geq w_i$.~Therefore, a direct usage
of dual coordinate descent is not sufficient.  In order to make the
algorithm convergent with minimal modification while retaining its
iterative message-passing nature, we use barrier (penalty) function
based approach. With an appropriate choice of barrier and using result of
Luo and Tseng \cite{LT91}, we will find the new algorithm to be
convergent.

To this end, consider the following convex optimization problem
obtained from \DL~ by adding a logarithmic barrier for constraint
violations with $\beps \geq 0$ controlling penalty due to violation.
Define
$$ g(\beps, \lambda) = \lf(\sum_{(i,j) \in E}   \lambda_{ij}\rf) - \beps \lf(\sum_{i\in V} \log \lf[\sum_{j \in \cN(i)} \lambda_{ij} - w_i\rf] \rf).$$
Then, the modified \DL~optimization problem becomes
\begin{eqnarray*}
\text{\CPe} : & & {\sf min} ~~ g(\beps, \lambda) \\
{\sf subject ~to} & &  \lambda_{ij} \geq 0, ~\mbox{for all} ~(i,j) \in E. 
\end{eqnarray*}
The algorithm \DESCENT($\beps,\delta$)~ is  coordinate descent on
\CPe, to within tolerance $\delta$, implemented via passing messages
between nodes. We describe it in detail as follows.

\vspace{.05in}
\noindent{\bf \DESCENT($\beps, \delta$)}
\vspace{.05in}
\hrule
\vspace{.1in}
\begin{itemize}
\item[(o)] The parameters are variables $\lambda_{ij}$, one for each
  edge $(i,j)\in E$. We will use notation that $\lambda^t_{ij} =
  \lambda^t_{ji}$. The vector $\lambda$ is iteratively updated, with
  $t$ denoting the iteration number.
  \begin{itemize}
  \item[$\circ$] Initially, set $t = 0$ and $\lambda^0_{ij} = \max\{w_i,
    w_j\}$ for all $(i,j) \in E$.  
  \end{itemize}
\item[(i)] In iteration $t+1$, update parameters as follows:
  \begin{itemize}
  \item[$\circ$] Pick an edge $(i,j) \in E$. The edge selection is
    done in a round-robin manner over all edges. 
  \item[$\circ$] For all $(i', j') \in E, (i',j') \neq (i,j)$ do nothing, i.e.
    $ \lambda^{t+1}_{i'j'} = \lambda^t_{i'j'}.$
  \item[$\circ$] For edge $(i,j)$, nodes $i$ and $j$ exchange messages
    as follows:
    $$ \gamma_{i\rightarrow j}^{t+1} = \left ( w_i - \sum_{k\neq j, k
    \in \cN(i)} \lambda_{ki}^t \right )_+,$$ 
    $$ \gamma_{j\rightarrow
    i}^{t+1} = \left ( w_j - \sum_{k'\neq i, k' \in \cN(j)}
    \lambda_{k'j}^t \right )_+. $$
  \item[$\circ$] Update $\lambda^{t+1}_{ij}$ as follows: with $a =
    \gamma_{i\rightarrow j}^{t+1}$ and $b=\gamma_{j\rightarrow i}^{t+1}$,
    \begin{equation}\label{eq:lambda_update}
      \lambda_{ij}^{t+1} ~ = ~ \left(\frac{a+b+2\beps + \sqrt{(a-b)^2
	  + 4\beps^2}}{2}\right)_+.
    \end{equation}
  \end{itemize} 
\item[(ii)] Update $t = t+1$ and repeat till algorithm converges
  within $\delta$ for each component.
\item[(iii)] Output the vector $\lambda$, denoted by $\lambda^{\beps, \delta}$,
when the algorithm stops.
\end{itemize}
\vspace{.05in}
\hrule
\vspace{.1in}

\noindent{\bf Remark.} The updates in \DESCENT~ above are obtained by
small -- but important -- perturbation of standard dual coordinate
descent (\ref{eq:dco}). To see this, consider the iterative step in
(\ref{eq:lambda_update}).  First, note that
\begin{eqnarray*}
\frac{a+b+2\beps + \sqrt{(a-b)^2 + 4\beps^2}}{2} & > & \frac{a+b+2\beps + \sqrt{(a-b)^2}}{2} \\
& = & \frac{a+b + |a-b|+2\beps}{2} \\
& = & \max(a,b) + \beps.
\end{eqnarray*}
Similarly, 
\begin{eqnarray*}
& & \frac{a+b+2\beps + \sqrt{(a-b)^2 + 4\beps^2}}{2} \\ 
& & \qquad\qquad ~\leq ~ \frac{a+b+2\beps + \sqrt{(a-b)^2 + 4\beps (a-b) + 4\beps^2}}{2} \\
& & \qquad \qquad ~=~ \frac{a+b + |a-b|+4\beps}{2} \\
& & \qquad \qquad ~=~  \max(a,b) + 2\beps.
\end{eqnarray*}
Therefore, we conclude that (\ref{eq:lambda_update}) can be re-written
as
\begin{eqnarray*} 
\lambda_{ij}^{t+1} & = & \beta \beps + ~ \max\left\{-\beta \beps, \lf (  w_i
- \sum_{k\in \cN(i)\backslash j} \lambda^t_{ik}\rf ), \rf. \\
& ~& \qquad \qquad \qquad \qquad \lf. \lf (~ w_j - \sum_{k \in \cN(j)\backslash i} \lambda^t_{kj} \rf
) \right\}, 
\end{eqnarray*} 
where for some $\beta \in (1,2]$ with its precise value dependent on
$\gamma^{t+1}_{i\to j}, \gamma^{t+1}_{j\to i}$.  This small
perturbation takes $\lambda$ close to the true dual optimum. In
practice, we believe that instead of calculating exact value of
$\beta$, use of some arbitrary $\beta \in (1,2]$ should be sufficient.

\subsection{\DESCENT: properties}

The \DESCENT~algorithm finds a good approximation to an optimum of
\DL, for small enough $\beps, \delta$. Furthermore, it always
converges, and does so quickly. The following lemma specifies the
convergence and correctness guarantees of \DESCENT.
\begin{lemma}\label{lem:thm0}
  For given $\beps, \delta > 0$, let $\lambda^t$ be the parameter
  value at the end of iteration $t \geq 1$ under \DESCENT($\beps,
  \delta$).  Then, there exists a unique limit point $\lambda^{\beps,
    \delta}$ such that
  \begin{eqnarray}
    \|\lambda^t - \lambda^{\beps, \delta} \| & \leq&  A \exp\lf(-Bt\rf), \label{eq:lt0}
  \end{eqnarray}
  for some positive constant $A, B$ (which may depend on problem
  parameter, $\beps$ and $\delta$). Let $\lambda^\beps$ be the
  solution of \CPe.~Then,
  $$\lim_{\delta\to 0} \lambda^{\beps, \delta} = \lambda^\beps.$$
  Further, by taking $\beps \to 0$, $\lambda^{\beps}$ goes to
  $\lambda^*$, an optimal solution to the \DL.~
\end{lemma}
We first discuss the proofs of two facts in Lemma \ref{lem:thm0}: (a)
$\lim_{\delta\to 0} \lambda^{\beps, \delta} = \lambda^\beps$ is a
direct consequence of the fact that if we ran \DESCENT~ algorithm with
$\delta =0$, it converges; (b) the fact that as $\beps \to 0$,
$\lambda^{\beps}$ goes to a dual optimal solution $\lambda^*$ follows
from \cite[Prop. 4.1.1]{NLP}. Now, it remains to establish the
convergence of the \DESCENT($\beps,\delta$) algorithm.  This will
follow as a corollary of result by Luo and Tseng \cite{LT91}. In order
to state the result in \cite{LT91}, some notation needs to be
introduced as follows.

Consider a real valued function $\phi: \RR^n \to \RR$ defined as
$$ \phi(\unx) = \psi(E\unx) + \sum_{i=1}^n w_i x_i, $$
where $E \in \RR^{m\times n}$ is an $m\times n$
matrix with no zero column (i.e., all coordinates of $\unx$
are useful), ${\mathbf w} = (w_i) \in \RR^n$ is
a given fixed vector, and $\psi:\RR^m \to \RR$  is
a strongly convex function on its domain
$$D_\psi = \left\{ \uny \in \RR^m : \psi(\uny) \in (-\infty, \infty) \right\}.$$
We have $D_\psi$ being open and let $\partial D_\psi$ denote
its boundary. We also have that,
along any sequence $\uny_k$ such that $\uny_k \to \partial D_\psi$
(i.e., approaches boundary of $D_\psi$), $\psi(\uny^k) \to \infty$.
The goal is to solve the optimization problem
\begin{eqnarray}
\mbox{minimize} & ~ & \phi(\unx) \nonumber \\
\mbox{over} & ~& \unx \in {\cal X}. \label{eq:cd1}
\end{eqnarray}
In the above, we assume that ${\cal X}$ is box-type, i.e.,
$$ {\cal X} = \prod_{i=1}^n [\ell_i, u_i], ~~\ell_i, u_i \in \RR.$$
Let ${\cal X}^*$ be the set of all optimal solutions of the problem
(\ref{eq:cd1}).
The ``round-robin'' or ``cyclic'' coordinate descent algorithm (the
one used in \DESCENT) for this problem has the following 
convergence property, as proved in Theorem 6.2~\cite{LT91}.
\begin{lemma}\label{lem:cd}
There exist constants $\alpha'$ and $\beta'$ which
may depend on the problem parameters in terms of $g, E, \underline{w}$
such that starting from the initial value $\unx^0$, we have in
iteration $t$ of the algorithm
$$ d(\unx^t, {\cal X}^*) \leq \alpha' \exp\left(- \beta' t\right) d(\unx^0, {\cal X}^*).$$
Here, $d(\cdot, {\cal X}^*)$ denotes distance to the optimal set ${\cal X}^*$.
\end{lemma}

\vspace{.05in}

\begin{proof}[Proof of Lemma \ref{lem:thm0}]
It suffices to check that the conditions assumed in 
the statement of Lemma \ref{lem:cd} apply in our set up
of Lemma \ref{lem:thm0} in order to complete the proof.

Note first that the constraints $\lambda_{ij}\geq 0$ in \CPe~ are of
``box-type'', as required by Lemma \ref{lem:cd}. Now, we need to show
that $g(\cdot)$ satisfies the conditions that $\phi(\cdot)$ satisfied
in (\ref{eq:cd1}). By observation, we see that the linear part in
$g(\cdot)$ is $\sum_{ij} \lambda_{ij}$ corresponds to the linear part
in $\phi$. Now, the other part in $g(\cdot)$, which corresponds to
$h(\beps,\lambda)$ where define
$$ h(\beps,\lambda) = -\beps \sum_{i} \log (\sum_{j\in \cN(i)}
\lambda_{ij} - w_i).$$ By definition, the $h(\cdot)$ is strictly
convex on its domain which is an open set as for any $i$, if
$$\sum_{j\in \cN(i)} \lambda_{ij} \downarrow w_i,$$ then $h(\cdot)
\uparrow \infty$. Note that for $h(\cdot) \to \infty$ towards boundary
corresponding to $\|\lambda\| \to \infty$ can be adjusted by
redefining $h(\cdot)$ to include some parts of the linear term in
$g(\cdot)$. Finally, the condition corresponding to $E$ not having any
zero column in (\ref{eq:cd1}) follows for any connected graph, which
is of our interest here. Thus, we have verified conditions of Lemma
\ref{lem:cd}, and hence established the proof of (\ref{eq:lt0}).  This
completes the proof of Lemma \ref{lem:thm0}.
\end{proof}

\subsection{\EST: algorithm}

The algorithm \DESCENT~yields a good approximation of the optimal
solution to \DL,~for small values of $\beps$ and $\delta$.  However,
our interest is in the (integral) optimum of \LP, when it exists.
There is no general procedure to recover an optimum of a linear
program from an optimum of its dual. However, we show that such a
recovery {\em is} possible through our algorithm, called \EST~ and
presented below, for the MWIS problem when $G$ is bipartite with a
unique MWIS.  This procedure is likely to extend for general $G$ when
\LP~relaxation is tight and \LP ~has a unique solution. In the
following $\delta_1$ is chosen to be an appropriately small number,
and $\lambda$ is expected to be (close to) a dual optimum.

\vspace{.05in}
\noindent{\bf \EST($\lambda, \delta_1$).}
\vspace{.05in}
\hrule
\vspace{.1in}

\begin{itemize}

\item[(o)] The algorithm iteratively estimates $\bx = (x_i)$ given
$\lambda$ (expected to be a dual optimum).

\item[(i)] Initially, color a node $i$ {\em gray} and set $x_i=0$ if
$\sum_{j \in \cN(i)} \lambda_{ij} > w_i + \delta_1$.  Color all other
nodes with {\em green} and leave their values unspecified.

\item[(ii)] Repeat the following steps (in any order) until no more
changes can happen:
\begin{itemize}
\item[$\circ$] if $i$ is {\em green} and there exists a {\em gray}
node $j \in \cN(i)$ with $\lambda_{ij} > \delta_1$, then set $x_i = 1$
and color it {\em orange}.

\item[$\circ$] if $i$ is {\em green} and some {\em orange} node $j \in
 \cN(i)$, then set $x_i = 0$ and color it {\em gray}.
\end{itemize}

\item[(iii)] If any node is {\em green}, say $i$, set $x_i =1$ and
color it {\em red}.

\item[(iv)] Produce the output $\bx$ as an estimation.

\end{itemize}

\vspace{.05in}
\hrule
\vspace{.05in}

\subsection{\EST: properties}

\begin{lemma}\label{lem:thm1}
Let $\lambda^*$ be an optimal solution of \DL.~If $G$ is a bipartite
graph with unique MWIS, then the output produced by
\EST($\lambda^*,0$) is the maximum weight independent set of $G$.
\end{lemma}
\begin{proof}

Let $x$ be output of \EST($\lambda^*,0$), and $\bx^*$ the unique
optimal MWIS. To establish $x = \bx^*$, it is sufficient to establish
that $x$ and $\lambda^*$ together satisfy the complimentary slackness
conditions stated in  Lemma \ref{lem:cs}, namely
\begin{enumerate}
\item[(x1)] $x_i (\sum_{j \in \cN(i)} \lambda^*_{ij} - w_i) = 0$ for
all $i \in V$, 
\item[(x2)] $(x_i + x_j - 1) \lambda^*_{ij} = 0$ for all $(i,j) \in
  E$, and 
\item[(x3)] $x$ is a feasible solution for the \IP.
\end{enumerate}

From the way the color {\em gray} is assigned initially, it follows
that either $x_i = 0$ or $\sum_j \lambda_{ij} - w_i = 0$ for all nodes
$i$. Thus (x1) is satisfied.


Before proceeding we note that all nodes initially colored gray are
correct, i.e. $x_i = x_i^* = 0$; this is because the optimal $\bx^*$
satisfies (x1). Now consider any node $j$ that is colored orange due
to there being a neighbor $i$ that is one of the initial grays, and
$\lambda_{ij} > 0$. For this node we have that $x_j = x^*_j = 1$,
because $\bx^*$ satisfies (x2). Proceeding in this fashion, it is easy
to establish that all nodes colored gray or orange are assigned values
consistent with the actual MWIS $\bx^*$.

Now to prove (x2); consider a particular edge $(i,j)$. For this, if
$\lambda^*_{ij} = 0$ then the (x2) is satisfied. So suppose
$\lambda^*_{ij} > 0$, but $x_i + x_j \neq 1$. This will happen if both
$x_i = x_j = 0$, or both are equal to 1. Now, both are equal to 0 only
if they are both colored gray, in which case we know that the actual
optima $x_i^* = x_j^* = 1$ as well. But this means that (x2) is
violated by the true optimum $\bx^*$, which is a contradiction. Thus
it has to be that $x_i = x_j = 1$ for violation to occur. However,
this is also a violation of (x3), namely the feasibility of $x$ for
the IP. Thus all that remains to be done is to establish (x3).

Assume now that (x3) is violated, i.e. there exists a subset $E'$ of
the edges whose both endpoints are set to 1. Let $S_1 \subset V_1, S_2
\subset V_2$ be these endpoints. Note that, by assumption, $S_1 \neq
\emptyset, S_2 \neq \emptyset$. We now use $S_1$ and $S_2$ to
construct two distinct optima of \IP, which will be a violation of our
assumption of uniqueness of the MWIS. The two optima, denoted
$\hat{x}$ and $\tilde{x}$, are obtained as follows: in $x$, modify
$x_i = 0$ for all $i \in S_1$ to obtain $\hat{x}$; in $x$ modify $x_i
= 0$ for all $i \in S_2$ to obtain $\tilde{x}$. We now show that both
$\hat{x}$ and $\tilde{x}$ satisfy all three conditions (x1), (x2) and
(x3).

Recall that the nodes in $S_1$ and $S_2$ must have been colored {\em
red} by the algorithm \EST.~Now, we establish optimality of $\hat{x}$
and $\tilde{x}$.  By construction, both $\hat{x}$ and $\tilde{x}$
satisfy (x1) since we have only changed assignment of {\em red} nodes
which were not {\em binding} for constraint (x1).

Now, we turn our attention towards (x2) and (x3) for $\hat{x}$ and
$\tilde{x}$. Again, both solutions satisfy (x2) and (x3) along 
edges $(i, j) \in E$ such that $i \in S_1, j \in S_2$ or else
they would not have been colored {\em red}. 
By construction, they satisfy (x3) along all other edges as well.
Now we show that $\hat{x}, \tilde{x}$ satisfy (x2) along 
edges  $(i,j) \in E$, such that $i \in S_1, j \notin S_2$ or 
$i \notin S_1, j \in S_2$. For this, we claim that all such
edges must have  $\lambda^*_{ij} = 0$: if not, that is
$\lambda^*_{ij} > 0$, then either $i$ or $j$ must have been
colored {\em orange} and an {\em orange} node can not be part 
of $S_1$ or $S_2$. Thus, we have established that both $\hat{x}$ 
and $\tilde{x}$ along with $\lambda^*$ satisfy (x1), (x2) and
(x3). The contradiction is thus established.

Thus, we have established that $x$ along with $\lambda^*$ satisfies
(x1), (x2) and (x3). Therefore, $x$ is the optimal solution of \LP,
and hence of the \IP. This completes the proof.
\end{proof}

Now, consider a version of \EST~ where we check for updating nodes in
a round-robin manner. That is, in an iteration we peform $O(n)$ operations.
Now, we state a simple bound on running time of \EST. 
\begin{lemma}\label{lem:thm1T}
The algorithm \EST~ stops after at most $O(n)$ iterations. 
\end{lemma}
\begin{proof}
The algorithm stops after the iteration in which no more 
node's status is updated. Since each node can be updated
at most once, with the above stopping condition an algorithm
can run for at most $O(n)$ iterations. This completes the
proof of Lemma \ref{lem:thm1T}.
\end{proof}

\subsection{Overall algorithm: convergence and correctness}

Before stating convergence, correctness and bound on convergence time
of the \AL($\beps,\delta,\delta_1$) algorithm, a few remarks are in
order. We first note that both \DESCENT~and \EST~are iterative
message-passing procedures. Second, when the MWIS is unique,
\DESCENT~need not produce an exact dual optimum for \EST~to obtain the
correct answer.  Finally, it is important to note that the above
algorithm always converges quickly, but may not produce good estimate
when \LP~relaxation is not tight. Next, we state the precise statement
of this result.
\begin{theorem}[Convergence \& Correctness]\label{thm:main1}
The algorithm \AL($\beps, \delta,\delta_1$) converges for any choice
of $\beps, \delta > 0$ and for any $G$. The solution obtained by it is
correct if $G$ is bipartite, \LP~has unique solution and $\beps,
\delta > 0, \delta_1$ are small enough.
\end{theorem}
\begin{proof}
The claim that algorithm \AL($\beps, \delta,\delta_1$) converges for
all values of $\beps, \delta, \delta_1$ and for any $G$ follows 
immediately from Lemmas \ref{lem:thm0}, \ref{lem:thm1} and \ref{lem:thm1T}.
Next, we worry about the correctness property. 

The Lemma \ref{lem:thm0} implies that for $\delta \to 0$, the output
of \DESCENT($\beps,\delta$), $\lambda^{\beps,\delta} \to
\lambda^\beps$, where $\lambda^\beps$ is the solution of \CPe.~Again,
as noted in Lemma \ref{lem:thm0}, $\lambda^{\beps} \rightarrow
\lambda^*$ as $\beps \rightarrow 0$, where $\lambda^*$ is an optimal
solution\footnote{There may be multiple dual optima, and in this case
$\lambda^\beps$ may not have a unique limit. However, every limit
point will be a dual optimum. In that case, the same proof still
holds; we skip it here to keep arguments simple.} of the
\DL.~Therefore, given $\delta > 0$, for small enough $\beps > 0$ we
have
\[
\left | \lambda^{\star,\beps}_{ij} - \lambda^*_{ij} \right | ~ \leq ~
  \frac{\delta}{3n} \quad \text{for all} ~ (i, j) \in E.
\]
We will suppose that the $\beps$ is chosen such.  As noted in the
earlier the algorithm converges for all choices of $\beps$. Therefore,
by Lemma \ref{lem:thm0} there exists large enough $T$ such that 
for $t \geq T$, we have 
\[
\left | \lambda^{\star,\beps}_{ij} - \lambda^t_{ij} \right | ~ \leq ~ \frac{\delta}{3n} \quad
  \text{for all} ~(i,j) \in E.
\]
Thus, for $t \geq T$ we have \beq \left | \lambda^{*}_{ij} -
\lambda^t_{ij} \right | ~ & \leq & ~ \frac{2\delta}{3n} \quad
\text{for all} ~(i,j) \in E. \label{xx1} \eeq

Now, recall Lemma \ref{lem:thm1}. It established that the
\EST($\lambda^*,0$) produces the correct max. weight independent set
as its output under hypothesis of Theorem \ref{thm:main1}.  Also
recall that the algorithm \EST($\lambda^*,0$) checks two conditions:
(a) whether $\lambda^*_{ij} > 0$ for $(i,j) \in E$; and (b) whether
$\sum_{j\in \cN(i)} \lambda^*_{ij} > w_i$. Given that the number of
nodes and edges are finite, there exists a $\delta$ such that (a) and
(b) are robust to noise of $\delta/n$. Therefore, by selection of
small $\delta_1$ for such choice of $\delta$, we find that the output
of \EST($\lambda^t,\delta_1$) algorithm will be the same as that of
\EST($\lambda^*,0$).  This completes the proof.
\end{proof}

\section{MAP Estimation as an MWIS Problem}\label{sec:map_as_mwis}

In this section we show that any MAP estimation problem is equivalent
to an MWIS problem on a suitably constructed graph with node
weights. This construction is related to the ``overcomplete basis''
representation \cite{WJ03}.  Consider the following canonical MAP
estimation problem: suppose we are given a distribution $q(\by)$ over
vectors $\by = (y_1,\ldots,y_M)$ of variables $y_m$, each of which can
take a finite value. Suppose also that $q$ factors into a product of
strictly positive functions, which we find convenient to denote in
exponential form:
\[
q(\by) ~ = ~ \frac{1}{Z} \prod_{\alpha \in A} \exp \lf (
\phi_\alpha(\by_\alpha) \rf ) ~ = ~ \frac{1}{Z} \exp \lf (
\sum_{\alpha\in A} \phi_\alpha(\by_\alpha) \rf )
\]
Here $\alpha$ specifies the domain of the function $\phi_\alpha$, and
$\by_\alpha$ is the vector of those variables that are in the domain
of $\phi_\alpha$. The $\alpha$'s also serve as an index for the
functions. $A$ is the set of functions. The MAP estimation problem is
to find a maximizing assignment $\by^* \in \arg \max_\by q(\by)$.

We now build an auxillary graph $\tG$, and assign weights to its
nodes, such that the MAP estimation problem above is equivalent to
finding the MWIS of $\tG$.  There is one node in $\tG$ for each pair
$(\alpha,\by_\alpha)$, where $\by_\alpha$ is an {\em assignment}
(i.e. a set of values for the variables) of domain $\alpha$. We will
denote this node of $\tG$ by $\delta(\alpha,\by_\alpha)$.



There is an edge in $\tG$ between any two nodes
$\delta(\alpha_1,\by_{\alpha_1}^1)$ and
$\delta(\alpha_2,\by_{\alpha_2}^2)$ if and only if there exists a
variable index $m$ such that
\begin{enumerate}
\item $m$ is in both domains, i.e. $m\in \alpha_1$ and $m\in
  \alpha_2$, and
\item the corresponding variable assignments are different,
  i.e. $y^1_m \neq y^2_m$.
\end{enumerate}
In other words, we put an edge between all pairs of nodes that
correspond to {\em inconsistent} assignments.  Given this graph $\tG$,
we now assign weights to the nodes. Let $c>0$ be any number such that
$c+\phi_\alpha(\by_\alpha) > 0$ for all $\alpha$ and $\by_\alpha$. The
existence of such a $c$ follows from the fact that the set of
assignments and domains is finite. Assign to each node
$\delta(\alpha,\by_\alpha)$ a weight of $c+\phi_\alpha(\by_\alpha)$.
\begin{lemma}\label{lem:map2mwis}
Suppose $q$ and $\tG$ are as above.
(a) If $\by^*$ is a MAP estimate of $q$, let $\delta^* = \{
  \delta(\alpha,\by^*_\alpha)\,|\, \alpha\in A\}$ be the set of nodes
  in $\tG$ that correspond to each domain being consistent with
  $\by^*$. Then, $\delta^*$ is an MWIS of $\tG$.
(b) Conversely, suppose $\delta^*$ is an MWIS of $\tG$. Then, for
  every domain $\alpha$, there is exactly one node
  $\delta(\alpha,\by^*_\alpha)$ included in $\delta^*$.  Further, the
  corresponding domain assignments$\{\by^*_\alpha\,|\,\alpha\in A\}$
  are consistent, and the resulting overall vector $\by^*$ is a MAP
  estimate of $q$.
\end{lemma}
\begin{proof}
A {\em maximal} independent set is one in which every node is either
in the set, or is adjacent to another node that is in the set. Since
weights are positive, any MWIS has to be maximal. For $\tG$ and $q$
as constructed, it is clear that
\begin{enumerate}
\item If $\by$ is an assignment of variables, consider the
  corresponding set of nodes $\{\delta(\alpha,\by_\alpha)\,|\,
  \alpha\in A\}$. Each domain $\alpha$ has exactly one node in this
  set. Also, this set is an independent set in $\tG$, because the
  partial assignments $\by_\alpha$ for all the nodes are consistent
  with $\by$, and hence with each other. This means that there will
  not be an edge in $\tG$ between any two nodes in the set.
\item Conversely, if $\Delta$ is a maximal independent set in $\tG$,
  then all the sets of partial assignments corresponding to each node
  in $\Delta$ are all consistent with each other, and with a global
  assignment $\by$.
\end{enumerate}
There is thus a one-to-one correspondence between maximal independent
sets in $\tG$ and assignments $\by$. The lemma follows from this
observation. 
\end{proof}


\begin{example} Let $y_1$ and $y_2$ be binary variables with joint distribution 
\[
q(y_1,y_2) ~ = ~ \frac{1}{Z}\exp(\theta_1 y_1 + \theta_2 y_2 +
\theta_{12}y_1 y_2)
\]
where the $\theta$ are any real numbers. The corresponding $\tG$ is
shown in the Figure \ref{fig:0}. Let $c$ be any number such that $c+\theta_1$,
$c+\theta_2$ and $c+\theta_{12}$ are all greater than 0. The weights
on the nodes in $\tG$ are: $\theta_1 + c$ on node ``1'' on the left,
$\theta_2 + c$ for node ``1'' on the right, $\theta_{12}+c$ for the
node ``11'', and $c$ for all the other nodes. 
\end{example} 
\begin{figure}[htb]
\begin{center}
\epsfig{file=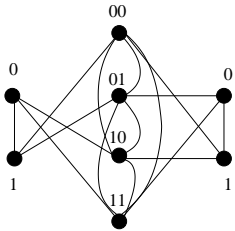,width=1.2in}
\end{center}
\label{fig:0}
\caption{An example of  reduction from MAP problem to max. weight
independent set problem.}
\end{figure}

\section{Discussion}

We believe this paper opens several interesting directions for
investigation.  In general, the exact relationship between max-product
and linear programming is not well understood. Their close similarity
for the MWIS problem, along with the reduction of MAP estimation to an
MWIS problem, suggests that the MWIS problem may provide a good first
step in an investigation of this relationship. 

Our novel message-passing algorithm and the reduction of MAP
estimation to an MWIS problem immediately yields a new message-passing
algorithm for general MAP estimation problem. It would be 
interesting to investigate the power of this algorithm on 
more general discrete estimation problems.



\end{document}